\documentclass[final, preprint]{article}

\PassOptionsToPackage{round}{natbib}

\usepackage{neurips_2022}

\usepackage[utf8]{inputenc} 
\usepackage[T1]{fontenc}    
\usepackage{graphicx}       
\usepackage{wrapfig}        
\usepackage{hyperref}       
\usepackage{url}            
\usepackage{booktabs}       
\usepackage{amsfonts}       
\usepackage{nicefrac}       
\usepackage{microtype}      
\usepackage{xcolor}         
\usepackage{amsmath}
\usepackage{amssymb}
\usepackage{todonotes}

\DeclareMathOperator*{\argmin}{arg\,min}

\usepackage{cleveref}
\Crefname{equation}{Eq.}{Eqs.}
\Crefname{figure}{Fig.}{Figs.}
\Crefname{tabular}{Tab.}{Tabs.}
\Crefname{section}{Sec.}{Secs.}
\Crefname{definition}{Def.}{Defs.}

\usepackage{amsthm}
\newtheorem{theorem}{Theorem}
\newtheorem{corollary}{Corollary}

\newtheorem{definition}{Definition}

\title{Symbolic Regression is NP-hard}

\author{%
  Marco Virgolin \\
  Centrum Wiskunde \& Informatica\\
  Amsterdam, the Netherlands \\
  \texttt{marco.virgolin@cwi.nl} \\
  \And
  Solon P.\ Pissis \\
  Centrum Wiskunde \& Informatica\\
  Amsterdam, the Netherlands \\
  \texttt{solon.pissis@cwi.nl} \\
}

\begin{document}

\maketitle

\begin{abstract}
  Symbolic regression (SR) is the task of learning a model of data in the form of a mathematical expression.
  By their nature, SR models have the potential to be accurate and human-interpretable at the same time.
  Unfortunately, finding such models, i.e., performing SR,
  appears to be a computationally intensive task.
  Historically, SR has been tackled with heuristics such as greedy or genetic algorithms
  and, while some works have hinted at the possible hardness of SR, no proof has yet been given that SR is, in fact, NP-hard.
  This begs the question: Is there an exact polynomial-time algorithm to compute SR models?
  We provide evidence suggesting that the answer is probably negative by showing that SR is NP-hard.
\end{abstract}

\section{Introduction}

Symbolic regression (SR) is a sub-field of machine learning concerned with discovering a model of the given data in the form of a mathematical expression (or equation)~\citep{koza1994genetic,schmidt2009distilling}.
For example, consider having measurements of planet masses $m_1$ and $m_2$, the distance $r$ between them, and the respective gravitational force $F$.
Then, an SR algorithm would ideally re-discover the well-known expression (or an equivalent formulation thereof) $F=G \times \frac{m_1 m_2}{r^2}$, with $G = 6.6743 \times 10^{-11}$, by opportunely combining the mathematical operations (here, of multiplication and division) with the variables and constant at play.

The appeal of learning models as mathematical expressions goes beyond obtaining predictive power alone, as is commonplace in machine learning.
In fact, SR models are particularly well suited for human interpretability and in-depth analysis~\citep{otte2013safe,virgolin2021model,la2021contemporary}.
This aspect enables a safe and responsible
use of machine learning models for high-stakes societal applications,
as requested in the AI acts by the European Union and the United States~\citep{euact,usact,jobin2019global}.
Moreover, it enables scientists to gain deeper knowledge about the phenomena that underlie the data.
Consequently, SR enjoys wide applicability: SR has successfully been applied to 
astrophysics~\citep{lemos2022rediscovering}, chemistry~\citep{hernandez2019fast}, control~\citep{derner2020constructing}, economics~\citep{verstyuk2022machine}, mechanical engineering~\citep{kronberger2018predicting}, medicine~\citep{virgolin2020machine}, space exploration~\citep{martens2022symbolic}, and more~\citep{matsubara2022rethinking}.

As we will describe in \Cref{sec:related}, many different algorithms have been proposed to address SR, ranging from genetic algorithms to deep learning ones. 
Existing algorithms either lack optimality guarantees or heavily restrict the space of SR models to consider.
In fact, there is a wide belief in the community that SR is an NP-hard problem\footnote{\cite{lu2016using} state that SR is NP-hard but provide no reference nor proof.}~\citep{lu2016using,petersen2019deep,udrescu2020ai,li2022console}.
However, to the best of our knowledge, this belief had yet to be solidified in the form of a proof prior to the advent of this paper.
Indeed, we prove that there exist instances of the SR problem for which one cannot discover the best-possible mathematical expression in polynomial time. 
Id est, SR is an NP-hard problem.

\section{Background: Existing SR algorithms}
\label{sec:related}
The introduction of SR is generally attributed to John R.~Koza (e.g., \cite{zelinka2005analytic} make this claim); however, the problem of finding a mathematical expression or equation that explains empirical measurements was already considered in earlier works~\citep{gerwin1974information,langley1981data,falkenhainer1986integrating}.
Such works build mathematical expressions by iterative application of multiple heuristic tests on the data.

\begin{wrapfigure}{r}{0.3\textwidth}
    \vspace{-.3cm}
    \centering
    \includegraphics[width=0.9\linewidth]{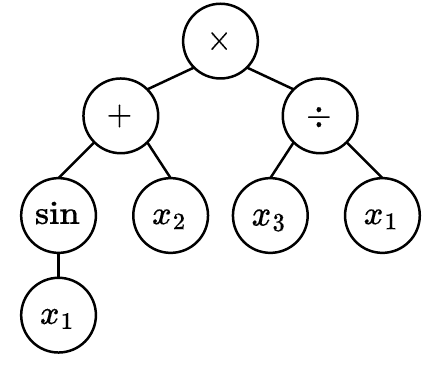}
    \caption{Example of a tree that encodes $f(\mathbf{x})=\left( \sin(x_1)+x_2 \right) \times x_3  / x_1$.}
    \label{fig:graph-sr}
    \vspace{-.25cm}
\end{wrapfigure}
Koza is best known for his pioneering work on genetic programming (GP), i.e., the form of evolutionary computation where candidate solutions are variable-sized and represent programs~\citep{koza1989hierarchical,koza1990genetic,koza1994genetic}.
Early forms of GP where proposed by~\cite{cramer1985representation,hicklin1986application}.
Koza showed that GP can be used to discover SR models by encoding mathematical expressions as computational trees (see \Cref{fig:graph-sr}).
In such trees, internal nodes represent functions (e.g., $+$, $-$, $\times$, etc.) that are drawn from a pre-decided set of possibilities, and leaf nodes represent variables or constants (e.g., $x_1$, $x_2$, \dots, $-1$, $\pi$, etc.).
GP evolves a population of trees by initially sampling random trees, and then conducts the following steps: (1) stochastic replacement and recombination of their sub-trees; (2) evaluation of the fitness by executing the trees and assessing their output; and (3) stochastic survival of the fittest. 

Recently, \cite{la2021contemporary} proposed \emph{SRBench}, a benchmarking platform for SR that includes more than $20$ algorithms have been applied on more than 250 data sets.
SRBench shows that several state-of-the-art algorithms for SR are GP-based.
Among these, at the time of writing, \emph{Operon} by~\cite{burlacu2020operon} was found to perform best in terms of discovering accurate SR models; and \emph{GP-GOMEA} by~\cite{virgolin2021improving} was found to perform best in terms of discovering decently-accurate and relatively-simple SR models (i.e., shorter mathematical expressions).
Other forms of GP, such as \emph{strongly-typed} GP~\citep{montana1995strongly}, \emph{grammar-guided} GP~\citep{mckay2010grammar}, and \emph{grammatical evolution}~\citep{o2001grammatical}, are often used to tackle \emph{dimensionally-aware} SR, i.e., the search of mathematical expressions with constraints to achieve meaningful combinations of units of measurement.

SR has been addressed with many other types of algorithms than genetic ones, oftentimes in order to obtain a deterministic behavior.
\cite{worm2013prioritized} and \cite{kammerer2020symbolic} proposed enumeration algorithms which make SR tractable by restricting the space of possible models to consider and including dynamic programming and pruning strategies.
\cite{cozad2014data,cozad2018global} showed how SR can be addressed with mixed integer nonlinear programming.
\cite{mcconaghy2011ffx} proposed \emph{FFX}, which generates a linear combination of many functions that are linearly-independent from each other, and then fits its coefficients with the \emph{elastic net}~\citep{zou2005regularization} to promote sparsity.
\cite{olivetti2018} and \cite{rivero2022dome} propose greedy algorithms that start from small mathematical expressions and iteratively expand them, by replacing existing components with larger ones from a set of possibilities.

Lastly, recent years have seen the proposal of deep learning-based algorithms for SR.
\cite{petersen2020deep} cast the SR problem as a reinforcement learning one and train a recurrent neural network to generate accurate SR models.
\cite{udrescu2020ai} leverage neural networks in order to test for symmetries and invariances in the data that are then used to prune the space of possible SR models.
An end-to-end approach is taken by~\cite{kamienny2022end} and~\cite{vastl2022symformer}, who train deep neural transformers to produce SR models directly from the data.
\cite{li2022console} seek SR models by proposing a convexified formulation of deep reinforcement learning.

In summary, existing SR algorithms are either heuristics, which do not guarantee optimality (e.g., genetic, greedy, or deep learning-based algorithms), or they are exact algorithms that achieve optimality but only over a small subset of all possible SR models, to limit the runtime (e.g., dynamic programming and mixed-integer nonlinear programming algorithms).
This strongly hints to the fact that SR is NP-hard.
As mentioned earlier, no proof has yet been given.

\section{Preliminaries}
\label{sec:preliminaries}

We will hereon refer to SR models as \emph{functions} when appropriate, as this is their fundamental nature. 
To begin, let us recall the concept of function composition, which is central to SR.
\begin{definition}
\emph{Function composition}

Given two functions $f : \mathbb{A} \rightarrow \mathbb{B}$ and $g : \mathbb{B} \rightarrow \mathbb{C}$, function composition, which we denote by $g \circ f$, is the operation that produces a third function $h : \mathbb{A} \rightarrow \mathbb{C}$, such that $h(x)=g(f(x))$.
\end{definition}

Thanks to function composition, we can now define the concept of \emph{search space} of an SR problem.

\begin{definition}
\label{def:searchspace}
\emph{Search space of SR}

Let $\mathcal{P}$ be a set of functions and variables.
The search space of SR is the function space $\mathcal{F}$ that contains all functions that can be formed by composition of the elements of $\mathcal{P}$ and their compositions.
\end{definition}

To better understand what \Cref{def:searchspace} states, consider that $\mathcal{P}$ can be set to contain a mix of functions that perform basic algebraic operations such as addition, subtraction, multiplication, and division; transcendental functions such as $\sin$, $\cos$, $\log$, $\exp$; constant functions (or simply \emph{constants}), such as $c_{42}(x)=42$ and $c_\pi(x)=\pi$ for any $x$; and variables of interest for the problem at hand, such as $x_1, x_2, x_3$.
$\mathcal{P}$ is typically referred to as the \emph{primitive set}, and its elements as \emph{primitives}~\citep{poli2008field}.
Once $\mathcal{P}$ has been decided, $\mathcal{F}$ is determined.
For example, choosing $\mathcal{P} = \{+(\cdot,\cdot), -(\cdot, \cdot), \times(\cdot,\cdot), x_1, x_2, -1, +1\}$ means that $\mathcal{F}$ will contain a subset of all possible polynomials of arbitrary degree in $x_1$ and $x_2$. 
In particular, $\mathcal{F}$ is a subset because only some coefficients can be expressed, by composing constants with addition, subtraction, and multiplication.

Let us clarify a point regarding constants in particular.
Normally, one would include constants considered to be relevant to the instance of SR at hand.
For example, if the unknown phenomenon for which an SR model is sought is suspected to have sinusoidal components, it may be advisable to include multiples of $\pi$ in $\mathcal{P}$.
Moreover, $\mathcal{P}$ can be set to contain special elements that represent probability distributions from which constants can be sampled (see the concept of \emph{ephemeral random constant} described by~\cite{koza1994genetic,poli2008field}).
We denote one such element by $\mathfrak{R}$ and, e.g., $\mathfrak{R}$ can be chosen to represent the uniform distribution between two numbers, or the normal distribution with a certain mean and variance.
When an SR algorithm picks $\mathfrak{R}$ from $\mathcal{P}$ to compose an SR model, a constant is sampled from the distribution identified by $\mathfrak{R}$.
Here (more specifically, in \Cref{th:corollary-const}) we will generously assume that any constant can be sampled directly from $\mathfrak{R}$, and therefore that there is no need for a real-valued optimizer to be part of the SR algorithm.
For example, having $\mathcal{P} = \{+(\cdot,\cdot), -(\cdot, \cdot), \times(\cdot,\cdot), x_1, x_2, \mathfrak{R}\}$ will mean that $\mathcal{F}$ contains \emph{all} polynomials of arbitrary degree in $x_1$ and $x_2$.

We can now proceed by providing a definition of the SR problem.
While this definition can be extended to other domains, we focus on handling real-valued numbers as the majority of the works takes place in this domain, and subsets thereof.
\begin{definition}
\label{def:sr}
\emph{Symbolic Regression (SR) problem}

Given a set $\mathcal{P}$ of functions and variables,
a metric $\mathcal{L} : \mathbb{R}^n \times \mathbb{R}^n \rightarrow \mathbb{R}$, 
vectors $\mathbf{x}_i = \left( x_{1,i}, \dots, x_{d,i} \right) \in \mathbb{R}^d$
and scalars $y_i \in \mathbb{R}$, for $i=1, \dots, n$,
the SR problem asks for finding a function $f^\star$ such that:


\begin{equation}
\label{eq:sr}
    f^\star = \argmin_{f \in \mathcal{F}} \mathcal{L} \left( \mathbf{y}, f(\mathbf{x}) \right),
\end{equation}

where $\mathcal{F}$ is the search space that is defined by $\mathcal{P}$.
\end{definition}

We provide some remarks concerning the proposed definition of the SR problem.
Firstly, let us map the objects provided in the definition to terms familiar to a machine learning audience. 
The pair $(\mathbf{x}_i, y_i)$ is normally what is referred to as \emph{observation}, \emph{data point}, \emph{example}, or \emph{sample}, where $x_{j,i}$ is the value of the $j$th \emph{feature} or \emph{variable} for the $i$th observation, and $y_i$ is the value of the \emph{label} or \emph{target variable} for the same observation.
The set that contains the observations upon which $\mathcal{L}$ is computed, i.e., $\mathcal{D}=\{ (\mathbf{x}_i, y_i) \}^n_{i=1}$, is called \emph{training set}.
Moreover, the metric $\mathcal{L}$ is called \emph{loss function}.

Normally, we actually desire $f$ to \emph{generalize} to \emph{new} (or also called \emph{unseen}) observations, i.e., observations which are similar to those in $\mathcal{D}$ but not exactly the same (they come from the same underlying and unknown distribution).
In other words, it is not sufficient that $f^\star$ is a best-possible function with respect to the training set, as the loss should remain minimal also for new observations that are not available to us.
Still, considering a ``pure optimization'' formulation, as given in \Cref{eq:sr}, can be considered to be a pre-requisite for being able to machine-learn accurate models from the data; in fact,  much literature that concerns the generation of provably-optimal models provides proofs with respect to the training set alone (see, e.g., results for decision trees~\citep{hu2019optimal}).
In a similar fashion, here we will consider the case of minimizing the loss with respect to the training set $\mathcal{D}$ and show that this is already problematic for any SR algorithm.

As loss function, we consider $\mathcal{L}$ to be a metric (i.e., distance) which operates between the output $f(\mathbf{x}_i)$ and the label $y_i$ across $i=1,\dots,n$.
Commonly-used loss functions such as the \emph{mean absolute error}, \emph{mean squared error}, and \emph{root mean squared error} fit this definition.
However, certain works include regularization terms in the loss function, such as $\lambda \times C(f)$, where $\lambda \in \mathbb{R}$ controls the regularization strength and $C : \mathcal{F} \rightarrow \mathbb{R}$ is a function of the complexity of $f$.
Typical goals of such regularization terms are improving generalization (by limiting effects akin to Runge's phenomenon~\citep{fornberg2007runge}) and improving the interpretability of $f$.
For the latter, implementations of $C$ range from weighed counting of the number of primitives that constitute $f$~\citep{ekart2001selection,hein2018interpretable}, to machine learning models trained from human feedback to predict $f$'s interpretability~\citep{virgolin2020learning,virgolin2021model}.
Here, for simplicity, we focus on $\mathcal{L}$ being a plain metric as stated in \Cref{def:sr} or,  equivalently put, we consider $\lambda=0$.

Lastly on \Cref{def:sr}, we consider only cases in which $f$ is \emph{not} a recursive function, which to the best of our knowledge is the case for the majority of the literature on SR. 
Recursive function discovery is an interesting topic in general (see, e.g.,~\cite{d2022deep}), but it is not interesting here because recursive functions can take exponential time to compute (consider, e.g., Fibonacci's sequence). 
Therefore, it is obvious that the SR problem cannot be solved in polynomial time if certain recursive functions can be considered.
Here, we will assume that computing $f(\mathbf{x})$ and $\mathcal{L}(\mathbf{y},f(\mathbf{x}))$ can be done in polynomial time.
Regarding $\mathcal{L}$, our assumption is met for all commonly-used metrics by which $\mathcal{L}$ is implemented (mean absolute error, mean squared error, variants thereof with margins, etc.). 
In fact, computing losses of such form takes $O(n)$ time, i.e., the runtime is linear in the number of observations.
Regarding the computation of $f(\mathbf{x})$, $f$ itself can be implemented as a directed acyclic graph, where nodes represent the functions and variables from $\mathcal{P}$, and edges represent compositions.
To compute $f(\mathbf{x})$, it suffices to visit each node of the graph for each observation, thus requiring $O(\ell \times n)$, where $\ell$ is the number of primitives in $f$. 
\Cref{fig:graph-sr} shows an example of such a graph, especially in the form of a \emph{tree}, which is perhaps the most common way of encoding mathematical expressions in SR (see, e.g., the SR algorithms benchmarked by~\cite{la2021contemporary}).

We conclude this section with the following important definition.

\begin{definition}
\label{def:sr-dec}
\emph{Decision version of the SR problem (SR-Dec)}

Given an SR instance and an $\epsilon \in \mathbb{R}^+_0$, SR-Dec outputs \textsf{YES} if and only if:

\begin{equation}
\label{eq:sr-dec}
    \exists f \in \mathcal{F} : \mathcal{L} \left( \mathbf{y}, f(\mathbf{x}) \right) \leq \epsilon.
\end{equation}

\end{definition}

Essentially, \Cref{def:sr-dec} is the problem of deciding whether there exists a function $f$ in the search space such that its loss is smaller than a chosen threshold $\epsilon$.

\section{The result}

We proceed directly by providing the main result of this paper.

\begin{theorem}
\label{th:sr-nphard}
The SR problem is NP-hard.
\end{theorem}

\begin{proof}
Let us begin by stating that SR-Dec is in NP.
Recall that the computations of $f(\mathbf{x})$ and $\mathcal{L}(\mathbf{y},f(\mathbf{x}))$ take polynomial time (see \Cref{sec:preliminaries}). 
Of course, the check $\leq \epsilon$ takes $O(1)$ time. 
Thus, if $f$ is guessed by an oracle, then we can provide an answer to SR-Dec in polynomial time.

We proceed by considering the unbounded subset sum problem (USSP).
USSP is a similar problem to the unbounded knapsack problem, where a same item can be put in the knapsack an arbitrary number of times, and the weight of an item corresponds exactly to the profit gained by including that item in the knapsack.
The decision version of USSP, USSP-Dec, is defined as follows.
Given $j=1, \dots, k$ ($k$ items), $w_j \in \mathbb{N}$ (weight of that item), and $t \in \mathbb{N}$ (the target), USSP-Dec asks:
\begin{equation}
\label{eq:ussp}
\exists \mathbf{m} : \sum_{j=1}^k w_j m_j = t?  
\end{equation}
where $m_j \in \mathbb{N}_0$ (multiplicity with which an item is picked).
USSP-Dec is known to be NP-complete~\citep{kellerer2004introduction}.

To prove that SR-Dec is NP-complete, we show that any instance of USSP-Dec can be reduced to some instance of SR-Dec in polynomial time. 
To this end, we will restrict SR-Dec as follows:
(1) We pick the set of primitives $\mathcal{P}$ to be $\mathcal{P}=\{+, x_1, \dots, x_d \}$;
(2) We set $\epsilon = 0$.
In other words, we set the search space $\mathcal{F}$ to contain only linear sums of the features in the data set $\mathcal{D}$, i.e., functions of the form $f(\mathbf{x}) = \sum^d_{j=1} x_j m_j$ with $m_j \in \mathbb{N}_0$.
SR-Dec will output \textsf{YES} if and only if there exists such a function in $\mathcal{F}$ that achieves zero loss, i.e., it perfectly interpolates all observations in $\mathcal{D}$.

Next, we craft $\mathcal{D}$ to have a single observation ($n=1$) and $k$ features ($d=k$).
For the only observation in $\mathcal{D}$ (dropping the index for the observation number, since there is only one), we set $x_1 = w_1, x_2 = w_2, \dots, x_k = w_k$, and $y = t$.

Then, the following holds:

\begin{align}
\label{eq:derivation}
    \emph{(\Cref{eq:sr-dec}) \ \ } & \exists f \in \mathcal{F} : \mathcal{L} \left( y, f(\mathbf{x}) \right) \leq \epsilon? \\
    \emph{(Choosing $\epsilon = 0$) \ \ } & \exists f \in \mathcal{F} : \mathcal{L} \left( y, f(\mathbf{x}) \right) \leq 0? \\
    \emph{($\mathcal{L} \left( y, f(\mathbf{x}) \right)=0 \iff f(\mathbf{x})=y$) \ \ } & 
    \exists f \in \mathcal{F} : f(\mathbf{x}) = y? \\
    \emph{(Equivalence $y=t$ due to $\mathcal{D}$) \ \ } & 
    \exists f \in \mathcal{F} : f(\mathbf{x}) = t? \\
    \emph{(Expanding $\mathcal{F}$ based on choice of $\mathcal{P}$) \ \ } &
    \exists f \in \left\{ \sum_{j=1}^d x_j m_j : m_j \in \mathbb{N}_0 \right\} : f(\mathbf{x}) = t? \\
    \emph{(Equivalence $x_j=w_j, d=k$ due to $\mathcal{D}$) \ \ } &
    \exists f \in \left\{ \sum_{j=1}^k w_j m_j : m_j \in \mathbb{N}_0 \right\} : f(\mathbf{x}) = t? \\
    \emph{(Re-formulating in terms of $\mathbf{m}$) \ \ } &
    \exists \mathbf{m} : \sum_{j=1}^k w_j m_j = t?   \label{eq:last-step} 
\end{align}
In other words, there exist some instances of SR-Dec that can be re-formulated as USSP-Dec (cfr.~\Cref{eq:ussp,eq:last-step}).
Now, since assembling $\mathcal{P}$ as stated above takes linear time in $k$, picking $\epsilon=0$ takes $\mathcal{O}(1)$ time, and constructing $\mathcal{D}$ as stated above takes linear time in $k$, then any instance of USSB-Dec can be reduced to some instance of SR-Dec in polynomial time: SR-Dec is NP-complete.

We conclude the proof with a \emph{reductio ab absurdum}.
Let us assume that there exists an algorithm to compute an optimal $f^\star$ for the SR problem (\Cref{def:sr}) in polynomial time.
An optimal $f^\star$ is the one for which the loss is minimal, which means that using $f^\star$ in \Cref{eq:sr-dec} allows us to immediately answer SR-Dec.
Since verifying that $\mathcal{L}(\mathbf{y},f^\star(\mathbf{x})) \leq \epsilon$ takes polynomial time, we conclude that if the SR problem can be solved in polynomial time, then we can also solve SR-Dec in polynomial time.
Therefore, the SR problem is NP-hard.

\end{proof}

We remark that, in the proof of \Cref{th:sr-nphard}, we construct $\mathcal{P}$ so as not to contain $\mathfrak{R}$ (nor any constant).
Some readers might disagree with this quite broad definition of SR.
In fact, some SR algorithms heavily rely on the presence of constants as well as on their optimization (e.g., \emph{FFX} by~\cite{mcconaghy2011ffx} and \emph{FEAT} by~\cite{la2018learning}).
Not allowing for arbitrary constants to be present in the functions of the search space might be seen as a violation of the very definition of SR.
In other words, some might think that $\mathcal{P}$ \emph{must} contain $\mathfrak{R}$.
We next show that SR remains NP-hard in this special case.

\begin{corollary}
\label{th:corollary-const}
The SR problem is NP-hard even when $\mathcal{P}$ must include $\mathfrak{R}$.
\end{corollary}

\begin{proof}
We follow a similar construction of the proof of \Cref{th:sr-nphard}.
Namely, the only difference from before is in the way we pick $\mathcal{P}$ and construct $\mathcal{D}$.
This time, we set $\mathcal{P}$ to additionally contain $\mathfrak{R}$, i.e., $\mathcal{P}=\{+,x_1,x_2,\dots,x_d,\mathfrak{R}\}$. 
This means that the function space $\mathcal{F}$ now contains functions of the form $f(\mathbf{x}) = c + \sum^d_{j=1} x_j m_j$ with $m_j \in \mathbb{N}_0$ and $c \in \mathbb{R}$ (sampled from $\mathfrak{R}$).
As to $\mathcal{D}$, we will now include two observations instead of a single one.
The first observation is set as before, i.e.,
$x_{1,1} = w_1, x_{2,1} = w_2, \dots, x_{k,1} = w_k$ ($d=k$) and $y_1 = t$.
As to the second observation, we set
$x_{1,2} = 0, x_{2,2} = 0, \dots, x_{k,2} = 0$ and $y_2=0$, i.e., the value of all features and of the label are set to zero. Now, $\mathcal{L}(\mathbf{y}, f(\mathbf{x}))=0 \iff f(\mathbf{x}_i) = y_i$ for \emph{both} $i=1,2$.
For $f(\mathbf{x}_2) = y_2 = 0$, since any $f$ has the form 
$f(\mathbf{x}) = c + \sum^d_{j=1} x_j m_j$ and $x_{j,2} = 0$, for all $j$, then $f(\mathbf{x}_2) = c + \sum^d_{j=1} 0 \times m_j = c$.
But $f(\mathbf{x}_2) = y_2 = 0 \iff c = 0$.
In other words, we know that every $f$ for which $c \neq 0$ is one for which SR-Dec outputs \textsf{NO}. 
Therefore, by construction, we can immediately ignore all of those functions, and consider only the subset of $\mathcal{F}$ that contains functions of the form $f(\mathbf{x}) = 0 + \sum^d_{j=1} x_j m_j = \sum^d_{j=1} x_j m_j$.
For every one of such functions, the loss for the second observation is by construction zero and we can therefore ignore it.
Consequently, we are now back to the same setting considered in \Cref{th:sr-nphard}, which concludes the proof.

\end{proof}

\section{Conclusion}
Our main contribution here was to prove that
symbolic regression (SR), i.e., the problem of discovering an accurate model of data in the form of a mathematical expression, is in fact NP-hard.
In particular, we have provided formal definitions of what SR should entail, and showed how the decision version of the unbounded subset sum problem can be reduced to a decision version of the SR problem. 
Except for the general definition of SR we considered, we have additionally shown that SR remains NP-hard even when the set of primitives must contain distributions from which constants can be sampled.

We hope that this note inspires more works on lower and upper bounds of different SR variants.

\section*{Acknowledgements}
We thank Marc Schoenauer (INRIA, France) for discussions on early works on symbolic regression.

\bibliographystyle{abbrvnat}
\bibliography{main}

\end{document}